\documentclass{article}
\usepackage{times}

\usepackage{soul}
\usepackage{url}
\usepackage[hidelinks]{hyperref}
\usepackage[utf8]{inputenc}
\usepackage[small]{caption}
\usepackage{graphicx}
\usepackage{amsmath}
\usepackage{amsthm}
\usepackage{booktabs}
\usepackage{ijcai21}
\usepackage{color}
\usepackage{thmtools}
\usepackage{thm-restate}

\usepackage{enumitem}
\setlist[itemize]{leftmargin=*}
\setlist[enumerate]{leftmargin=*}

\newtheorem{example}{Example}[section]

\title{Data-Driven Methods for Balancing Fairness and Efficiency in Ride-Pooling}

\newcommand{\Eqref}[1]{(\ref{#1})}

\author{
Naveen Raman$^1$\and
Sanket Shah$^2$\and
John P.\ Dickerson$^{1}$\\
\affiliations
$^1$University of Maryland\\
$^2$Harvard University\\
\emails
nraman1@umd.edu,
sanketshah@g.harvard.edu,
john@cs.umd.edu
}

\begin{document}
	\title{Data-Driven Methods for Balancing Fairness and Efficiency in Ride-Pooling}
	\maketitle
	\begin{abstract}
Rideshare and ride-pooling platforms use artificial intelligence-based matching algorithms to pair riders and drivers.
However, these platforms can induce inequality either through an unequal income distribution or disparate treatment of riders. 
We investigate two methods to reduce forms of inequality in ride-pooling platforms: (1) incorporating fairness constraints into the objective function and (2) redistributing income to drivers to reduce income fluctuation and inequality. 
To evaluate our solutions, we use the New York City taxi data set. 
For the first method, we find that optimizing for driver-side fairness outperforms state-of-the-art models on the number of riders serviced, both in the worst-off neighborhood and overall, showing that optimizing for fairness can assist profitability in certain circumstances. 
For the second method, we explore income redistribution as a way to combat income inequality by having drivers keep an $r$ fraction of their income, and contributing the rest to a redistribution pool.
For certain values of $r$, most drivers earn near their Shapley value, while still incentivizing drivers to maximize value, thereby avoiding the free-rider problem and reducing income variability.
The first method can be extended to many definitions of fairness and the second method provably improves fairness without affecting profitability. 
\end{abstract}

	\section{Introduction}\label{sec:intro}
Ride-pooling platforms, such as UberPool and Lyft-Line, manage independent drivers who can service multiple riders concurrently. 
To match riders and drivers, they use artificial-intelligence-based matching algorithms~\cite{turakhia12engineering}. 
While matching algorithms typically aim to maximize income, they can inadvertently have negative consequences on fairness, such as a gender wage gap~\cite{cook2018gender} or rider discrimination based on race~\cite{Brown18:Ridehail}. 

While the nascent literature on fairness in ride-pooling is limited, prior research has looked at maximizing the minimum utility across riders and drivers by approximating the problem as an instance of the bipartite matching problem~\cite{lesmana2019balancing}.
There have also been past papers that have looked into the sub-problem of fairness in \textit{rideshare}, where drivers can only service one rider request at a time (such as UberX), which simplifies the matching problem. 
Past work has used a bipartite matching problem framework to prove bounds on the trade-off between fairness and income~\cite{nanda2020balance,xu2020trade}. 
Additionally, past research into rideshare fairness has considered equalizing utility across riders and drivers~\cite{suhr2019two}. 

Recent work into the ride-pooling matching problem showed that using a Markov decision process (MDP) based approach, in combination with deep learning, maximizes the number of rider requests serviced because it makes non-myopic decisions~\cite{shah2019neural}. 
We build on prior ride-pooling work~\cite{lesmana2019balancing,shah2019neural} to develop a simple yet robust method to improve fairness non-myopically, which can be generalized to different notions of fairness. 
Motivated by reports of wage inequality and rider discrimination~\cite{Brown18:Ridehail,cook2018gender,moody2019rider}, we develop two methods to reduce certain definitions of inequality within the MDP framework: (1) modifying the objective function of the MDP and (2) creating an income redistribution method.\footnote{Our code and data is publicly available at \url{https://github.com/naveenr414/ijcai-rideshare}}  
Our contributions are the following:
\begin{enumerate}
    \item We extend an MDP-based framework to non-myopically optimize for different definitions of fairness.
    \item We propose new objective functions that aim to maintain profitability while minimizing inequality on the driver and rider-sides, and evaluate their effect on inequality. 
    We show that certain objective functions can reduce inequality while maintaining or even improving profitability. 
    \item 
    We implement an income redistribution scheme, where each driver contributes a certain percentage of income earned, which is then redistributed to other drivers to offset income fluctuation and reduce wage inequality. 
    We show that, under certain levels of risk tolerance, we can use income redistribution to significantly reduce wage inequality while avoiding the free-rider problem.
    Additionally, we prove that our income redistribution scheme guarantees drivers a minimum wage. 
\end{enumerate}

We find that varying the objective function can positively impact rider-side fairness, and utilizing income redistribution can improve driver-side fairness.
When used together, these methods can improve both rider and driver-side fairness, and could potentially be used not only for ride-pooling, but also for other matching and resource allocation problems.  	
	\section{Related Work}
The complexity of the ride-pooling matching problem led to the need for algorithmic solutions. 
One attempt to solve the problem reduces it to a version of the bipartite matching~\cite{zhao2019preference}.
Other papers model this problem as an MDP which takes into account future consequences of matching~\cite{lin2018efficient,li2019efficient}. 
Our work builds on previous work that uses offline-online learning and approximate dynamic programming to match drivers and riders non-myopically~\cite{shah2019neural}. 

While these algorithmic solutions serve more customers than traditional taxi services~\cite{uber2015chicago}, recent literature has raised questions about the fairness of the matches generated by these algorithms.  
On the rider side, Brown et al.~\cite{Brown18:Ridehail} highlight the disparate treatment of riders by ride-pooling companies, which results in higher rates of trip cancellation for black riders. 
Similarly, on the driver side, some ride-pooling drivers cannot achieve a living wage due to income inequality~\cite{graham2017towards}. 

One approach to dealing with these issues is to reformulate the problem using bipartite matching and utilize a min-max objective function to maximize the minimum utility for drivers and riders~\cite{lesmana2019balancing}. 
Within the sub-problem of rideshare matching, which is simpler because drivers cannot concurrently service multiple riders, past work has proven bounds on the trade-off between fairness and efficiency for a specific notion of fairness~\cite{nanda2020balance,xu2020trade,ma2020group}.
Past empirical research has looked into equalizing utility across drivers and riders~\cite{suhr2019two}, and improving the fairness of rideshare demand functions~\cite{yan2020fairness}. 

Our work builds on past ride-pooling work~\cite{shah2019neural,lesmana2019balancing} to develop a general method that can be adapted for many definitions of fairness, while also matching non-myopically by using an MDP.
We are the first study to look at fairness in MDP-based matching for ride-pooling or rideshare, and so our work applies to the broader class of non-bipartite matching problems. 
Additionally, we use income redistribution to reduce income fluctuation and inequality without affecting profitability, which has not been explored in previous work. 

	\section{Problem Statement}\label{sec:model}
In this section, we formally describe our problem,  methodology, and evaluation strategy.  
\subsection{Problem Description}
We consider the problem of matching a stream of rider requests to one of $n$ drivers. 
All riders and drivers reside on a graph, which consists of locations, $L$, and edges, $E$, where $E_{i,j}$ represents the travel time in minutes between location $i \in L$ and location $j \in L$. 
A trip between location $i$ and location $j$ is priced at $E_{i,j}+\delta$, where $\delta$ is a constant, capturing both the fixed and variable costs inherent in ride-pooling pricing.
Although the stream of rider requests is continuous, we batch requests and match once a minute, which emulates what ride-pooling companies perform in reality~\cite{matching}.
We define rider requests and driver states as follows: 
\begin{enumerate}
\item We define a rider request, which refers to a rider requesting a ride, to be the tuple  $u_i = (g_{i},e_{i},t_{i})$.
In the tuple, $g_{i}, e_{i} \in L$ are the starting and ending location for the rider request, and $t_{i}$ is the time when the rider request originated. 
\item We define the state of each of the $n$ drivers as $r_{i} = (m_{i},c_{i},d_{i},p_{i},s_{i})$, where $m_{i}$ is the capacity of driver $i$, and $c_{i}$ is the number of riders currently driven by driver $i$. 
$d_{i} \in L$ is the location of the driver, $p_{i}$ is the set of current requests that driver $i$ is servicing, and $s_{i}$ is the set of previously completed requests, making $|p_{i}| + |s_{i}|$ the total number of requests that will be serviced by driver $i$.
\end{enumerate}

\subsection{Matching Riders and Drivers}
We describe how we match $m$ rider requests, $U = u_{1} \cdots u_{m}$, with $n$ drivers, $R = r_{1} \cdots r_{n}$, using an MDP framework. 
Because each driver can be matched to multiple riders, we reduce the number of rider-driver combinations by generating feasible matchings for each driver. 
We let $F^{i}$ represent all feasible matchings for driver $i$, where each matching, $f \in F^{i}$ is a set of requests that can be served simultaneously by driver $i$~\cite{alonso2017demand}. 
We let $a^{i,f}$ denote an indicator value, determining whether the set of requests $f$ is matched to driver $i$. 
Each driver is matched to one $f \in F^{i}$, which could be the empty set, where the driver gets no new rider requests. 

The problem reduces to selecting feasible matchings that maximize an objective function, subject to constraints, which we solve through an integer linear program. 
Let the objective function be $o(R,W)$, which measures the benefit of having the drivers be in the state $R$, where $W$ is the set of all previous unaccepted and accepted requests.
If, after accepting a matching $f \in F^{i}$, the approximate new state of the drivers is $R'$, then the change in objective function is 
\begin{equation}
    \Delta o(R,W,f) = o(R',W \cup f)-o(R,W).
\end{equation} 
We approximate the post-acceptance state, $R'$, from prior work~\cite{shah2019neural}. 
To avoid matching myopically, we compute the value function, $V(R')$ for the state of drivers after matching. 
We approximate the value function through deep learning following past work, where the value function is an approximation of the objective function, $o(R,W)$, over a future state, $R'$, essentially stating the value of a driver residing in a state $R'$~\cite{shah2019neural}. 
To perform deep learning, we utilize a neural network that takes as inputs, the locations of the drivers and their current path, and uses that to learn the value function. 
The details of training, network structure, and configuring the deep network can be found in our code, and prior work~\cite{shah2019neural}.
We then weight each feasible matching as $\Delta o(R,W,f) + \gamma V(R')$, where the discount factor $\gamma=0.9$. 
Our aim is to maximize:
\begin{equation}\label{eq:lp-obj}
\sum_{i=1}^{n} \sum_{f \in F^{i}}^{} a^{i,f} (\Delta o(R,W,f) + \gamma V(R')) 
\end{equation}
While $o$ could be non-linear, we pre-compute $\Delta o(R,W,f)$ to avoid non-linearities.
We use~\Eqref{eq:lp-obj} as the objective function in an integer linear program subject to the following constraints:  
\begin{enumerate} 
    \item \textbf{Driver-side feasibility}. One assigned action per vehicle: $\forall i, \sum_{f \in F^{i}}^{} a^{i,f} = 1.$
    \item \textbf{Rider-side feasibility}. At most one assigned action per request:
    $ \forall j, \sum_{i=1}^{N} \sum_{f \in F^{i}, u_{j} \in f}^{} a^{i,f} \leq 1.$
\end{enumerate} 

Through the integer linear program (ILP), we compute $a^{i,f}$, which determines which feasible matchings are selected. 
The ILP is run online using an offline-trained value function. 

\subsection{Evaluation Strategy}
To evaluate our methods, we utilize the New York City taxi data set~\cite{NYYellowTaxi2016}, a commonly used dataset~\cite{lesmana2019balancing,shah2019neural,alonso2017demand} which contains pickup and drop-off locations and times for taxi passengers from March 23rd to April 1st and from April 4th to April 8th, 2016. 
Data from ride-pooling companies are generally proprietary, so we use taxi data to simulate ride-pooling under the assumption that the spatial and temporal distribution of rider requests are similar. 
The pickup and drop-off location for each rider request is reported in longitude-latitude coordinates; we discretize these into $|L|$ locations in New York City and compute travel time and paths between each pair of locations.  
The driver locations are initially randomized, but remain consistent between experiments and trials.
The data from both experiments (Sections~\ref{sec:fairness} and~\ref{sec:income}) are available online.\footnote{https://www.dropbox.com/s/y9pkmzmbjclrfrx/test-data.zip?dl=0}

	\section{Fairness-based Objective Functions}\label{sec:fairness}
We develop fairness based objective functions to improve both driver-side and rider-side fairness. 
In Section~\ref{sec:fairness-profit-metrics}, we discuss profitability metrics that prior work has optimized for. 
In Sections~\ref{sec:fairness-rider-metrics} and~\ref{sec:fairness-driver-metrics}, we define our notions of driver and rider-side fairness, and then operationalize each of these by developing two objective functions. 
In Section~\ref{sec:fairness-exp-setup} we discuss how we test each objective function, and in Sections~\ref{sec:fairness-exp-results} and~\ref{sec:fairness-discussion}, we discuss the results after running the experiments. 

\subsection{Profitability Metrics}\label{sec:fairness-profit-metrics}
We develop two different measures of profitability: the number of riders serviced and the total income accumulated by drivers. 
To define these two metrics, if the driver states are $R = r_{1} \cdots r_{n}$, where $r_{i} = (m_{i},c_{i},d_{i},p_{i},s_{i})$, then the total number of rides serviced by driver $i$ is $|p_{i}| + |s_{i}|$, which represents ongoing requests ($p_{i})$, and completed requests ($s_{i}$). 
We develop an objective function based on the total number of rider requests serviced by all drivers: \begin{equation} 
    o_{1}(R,W) = \sum_{i=1}^{n} |p_{i}| + |s_{i}|.
\end{equation}
Similarly, the income for any request $u=(g,e,t)$ is $E_{g,e} + \delta$. 
So, if we let $\pi_{i}$ denote the income for driver $i$, where 
\begin{equation} 
\pi_{i} = \sum_{u=(g,e,t) \in p_{i} \cup s_{i}} E_{g,e}+\delta,
\end{equation} then our income objective function is the total income: 
\begin{equation} 
    o_2(R,W) = \sum_{i=1}^{n} \pi_{i} .
\end{equation} 
The state-of-the-art method~
\cite{shah2019neural} matches based on maximizing the number of requests. 

\subsection{Rider-Side Fairness Metrics}\label{sec:fairness-rider-metrics}
We define fairness for both drivers and riders.  
We note that, as with all operationalizations of societally-relevant concepts such as fairness, the decision of \emph{which} definition of fairness to use---if one is appropriate to use at all---is morally-laden, and one that should be made with the explicit input of stakeholders.  
Our definitions are drawn in part from recent reports of (primarily rider-side) unfairness in fielded ride-pooling applications~\cite{Dillahunt17:Uncoering,Brown18:Ridehail,Pandey20:Iterative}; still, we acknowledge that these need not be one-size-fits-all formalizations of a complicated concept, and view them rather as illustrative examples of a class of fairness definitions that could be incorporated into automated matching algorithms used by ride-pooling platforms.

We base our conception of rider-side fairness on the differential treatment of riders based on membership in a protected group, such as race~\cite{Brown18:Ridehail}. 
Because we lack access to any protected information in our data set, we instead measure fairness based on the distribution of neighborhoods serviced, though the method could be applied to any protected group. 

We define a neighborhood as a set of locations in $L$, and so for any $l \in L$, we define the neighborhood function, $1 \leq N(l) \leq H$, as the function which maps locations to neighborhood labels. 
The neighborhood function is constructed through the $k$-means algorithm, which divides the $|L|$ locations into $H$ neighborhoods based on their latitude and longitude. 
To determine the distribution of neighborhoods serviced, we define $h_{j}$ as the number of requests serviced starting in neighborhood $j$, where $h_{j}$ is 
\begin{equation} 
    h_{j} = \sum_{i=1}^{n} \sum_{u=(g,e,t) \in (p_{i} \cup s_{i}), N(g)=j} 1,
\end{equation} 
The total number of requests, serviced or non-serviced, originating from neighborhood $j$ is $k_{j} \geq h_{j}$ defined as
\begin{equation}
    k_{j} = \sum_{u=(g,e,t) \in W, N(g)=j} 1. 
\end{equation}
Our aim is to equalize treatment across neighborhoods, and so we aim to equalize the percent of requests serviced, $\frac{h_{j}}{k_{j}}$, which we call the success rate. 
To avoid making all success rates 0, we include a profitability term, and regulate the size of the two terms with a hyperparameter $\lambda$. 
This makes our objective function for rider-side fairness: 
\begin{equation} 
o_3(R,W) = -\lambda\mathrm{Var}(\frac{h_{j}}{k_{j}}) + \sum_{i=1}^{n} \pi_{i}, 
\end{equation} where $\mathrm{Var}$ is the variance function. 
We develop two metrics of rider-side fairness which capture the spread and scale of disparate treatment: $\mathrm{min} (\frac{h_{j}}{k_{j}})$ and $\mathrm{Var}(\frac{h_{j}}{k_{j}})$ respectively. 

\subsection{Driver-Side Fairness Metrics}\label{sec:fairness-driver-metrics}
We develop corresponding objective functions and metrics on the driver-side. 
Recent news has noted wage discrepancies among drivers~\cite{graham2017towards,bokanyi2019ride}, so we aim to reduce income disparity. Similar to the rider-side, our objective function minimizes the spread of income while maintaining profitability, which is represented as \begin{equation} 
o_4(R,W) = - \lambda \mathrm{Var}(\pi_{i}) + \sum_{i=1}^{n} \pi_{i}.
\end{equation}
We measure the scale and spread of driver-side inequality through two metrics. 
The first measures the scale of income disparity by calculating the minimum income for any driver, which is $\mathrm{min}(\pi_{i})$, and the second captures the spread using the variance function, which is $\mathrm{Var}(\pi_{i})$

\subsection{Experiment Setup}\label{sec:fairness-exp-setup}
We compare the four objective functions, which we call request, income, rider-side fairness, and driver-side fairness. 
We utilize each objective function with the matching algorithm, train the corresponding value function, and test each objective function on data from April 4th.
The details of hyperparameters and value function training are in Appendix B. 

\begin{figure*}[t!]
\centering
    \includegraphics[width=0.49\linewidth]{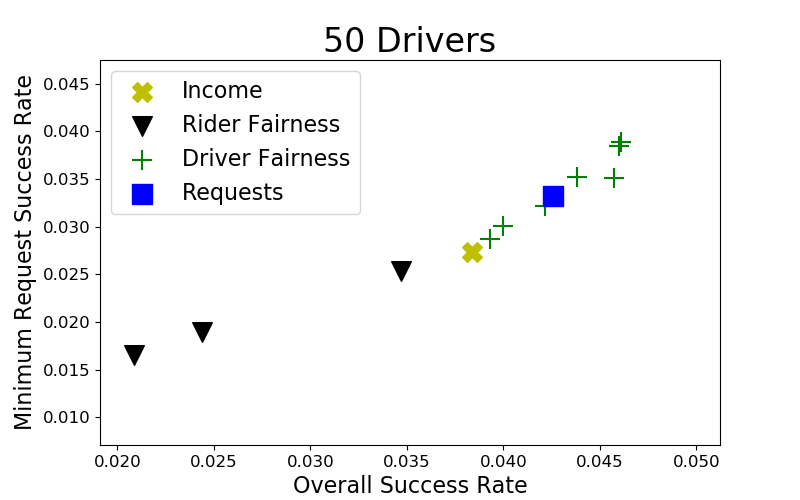}
    \includegraphics[width=0.49\linewidth]{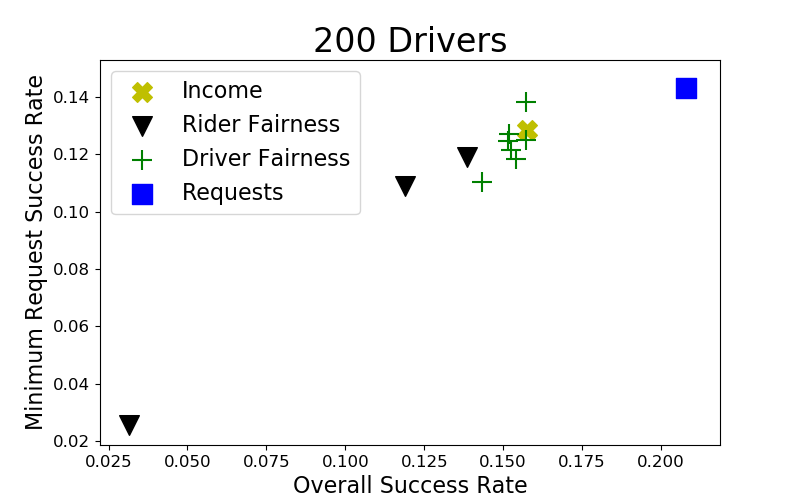}
    
    \caption{Each point represents one combination of hyperparameter and objective function. 
    We find that with $50$ total drivers, the objective function that minimizes the spread of income also achieves the highest request success rate both in the worst-off neighborhood and overall, outperforming the state-of-the-art objective function.   
    However, at $200$ drivers, the objective function which maximizes the number of requests serviced achieves the highest success rate overall and in the worst-off neighborhood.
    The success rates are small because the number of drivers is much smaller than the number of riders; as the number of drivers increases, the success rate will approach 1. 
    }
    \label{fig:service}
\end{figure*}
 
\begin{figure}[h!]
    \centering
\includegraphics[width=0.97\linewidth]{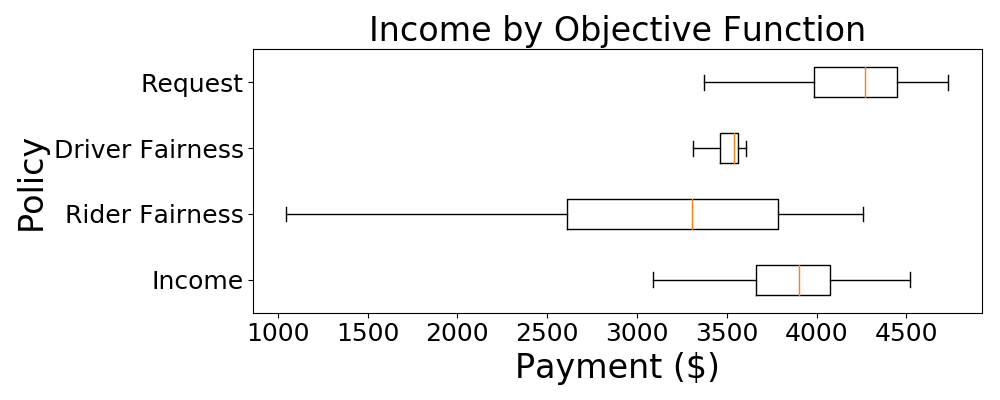}
    \caption{We compare the distribution of income for different objective functions, choosing a fixed value of $\lambda$
    , and set the number of drivers to $200$.  
    We find that optimizing for driver-side fairness reduces the income spread, however, at a cost to income for all drivers, as the income distribution shifts downwards. 
    Optimizing for rider-side fairness does no better, as the whole distribution again shifts downwards compared to the state-of-the-art method.}
    \label{fig:approach}
\end{figure}

\subsection{Experiment Results}\label{sec:fairness-exp-results}
We run experiments comparing the different objective functions on both profitability and fairness metrics. 
We describe the major conclusions from our experiments below: 
\begin{enumerate}
    \item \textbf{Alignment of fairness and profitability metrics.} We find that certain fairness metrics are aligned with profitability metrics. 
    In particular, objective functions that improve the overall success rate for riders also improve the success rate in the worst-off neighborhood (Figure~\ref{fig:service}). 
    Similarly, the objective function which maximizes total income at $200$ drivers, the request objective function, also maximizes income for those earning lower wages. 
    However, efforts to decrease the spread of income do not result in higher wages for those at the bottom, and instead lower wages across the board (Figure~\ref{fig:approach}). 
    \item \textbf{Driver-side fairness objective functions improve rider-side metrics.} We find that using an objective function that maximizes driver-side fairness manages to reduce the spread of income and also positively impacts rider-side inequality and profitability.
    In particular, at $50$ drivers, optimizing for driver-side fairness improves the overall success rate and the success rate in the worst-off neighborhood, but at $200$ drivers, the state-of-the-art objective function, which optimizes for the number of requests, maximizes both rider-side fairness and profitability metrics. 
    This is probably because, when optimizing for driver-side fairness, drivers tend to service lower value and shorter-distance rides to reduce the spread of income, allowing for more rides to get serviced. 
    However, when there are too many drivers, there are not sufficient short-distance requests. 
    \item \textbf{Inability to raise wages for lowest wage earners.} While objective functions that maximize driver-side fairness reduce the spread of income, they do so at a cost to profitability by lowering income across the board. 
    This is a general problem, as no objective function was able to raise wages for the lowest-earning drivers when compared to the state-of-the-art objective function. 
\end{enumerate}

\subsection{Discussion}\label{sec:fairness-discussion}
Our experiments show that, by changing the objective function used in matching, it is possible to improve rider-side metrics. 
Additionally, we find that certain fairness metrics, such as the success rate (fraction of requests serviced) in the worst-off neighborhood, are aligned with profitability metrics. 
On the other hand, metrics that measure spread, such as the spread of income, are not aligned with profitability metrics.
To tackle this issue, we introduce an income redistribution method (section 5) which tackles driver-side inequality without affecting profitability. 
The results from these experiments can be incorporated into other matching algorithms to improve fairness.
Additionally, our methods can be extended to other definitions of fairness by varying the objective function, and so are robust to different definitions of fairness. 

	\section{Income Redistribution}\label{sec:income}
To reduce the spread and fluctuation of income, we propose an income redistribution method, where each driver takes home a certain percentage of their income, and the rest is redistributed at the end of each day.
In Section~\ref{sec:income-model}, we formalize our income redistribution model, and we describe properties of the model in Section~\ref{sec:income-redistribution}. 
In Section~\ref{sec:income-exp-setup}, we outline experiments to test out income redistribution method, and we detail the results in Section~\ref{sec:income-exp-results}. 
\begin{figure*}[t!]
\includegraphics[width=0.4\linewidth]{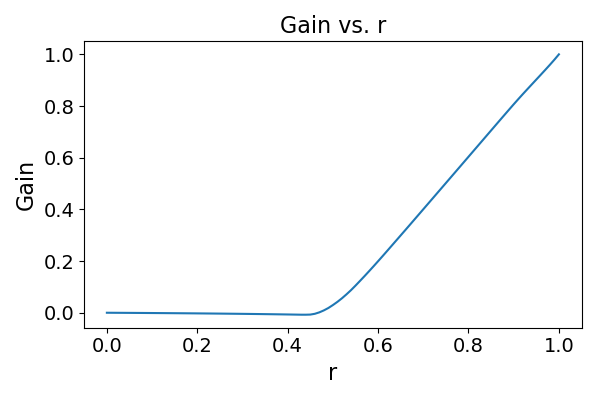}
\includegraphics[width=0.4\linewidth]{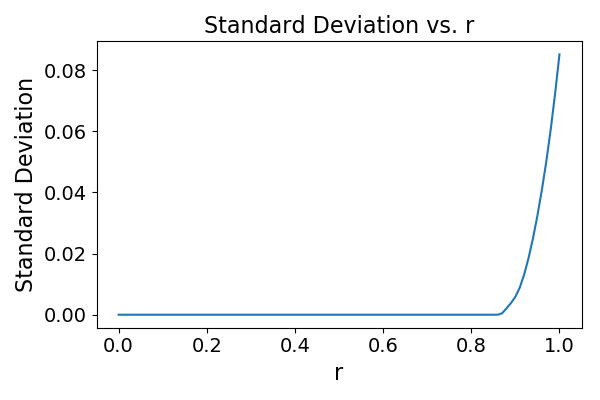}
\centering
\caption{Comparing the gain (\emph{left}) and standard deviation (\emph{right}) of income to value ratio for different values of $r$.  We find that when $0.5 \leq r \leq 0.9$, the gain is non-zero, while the standard deviation is small and non-increasing, meaning that drivers are incentivized to earn more without negatively affecting the income inequality.}%
\label{fig:shapley}
\end{figure*}

\subsection{Model of Income Redistribution}\label{sec:income-model}
Driver wages fluctuate day to day due to varying demand, making it difficult for drivers to earn a reliable wage. 
To counter this and reduce income inequality, we design an income redistribution scheme. 
To do this, we define the amount that drivers make before redistribution as $\pi_{i}, 1 \leq i \leq n$, and the amount they make after redistribution as $q_{i}$, which is dependent on the risk tolerance, $r$.  

More valuable drivers should have a higher $q_{i}$, as they might work longer hours or service higher need areas. 
To generate an objective measure of the value of a driver, we utilize a metric from game theory known as the Shapley value~\cite{Shapley53:Value}, which measures the marginal contribution of an agent. 
The Shapley value helps assess the true contributions of each driver, and was previously used to 
determine the value of data points~\cite{jia2019towards} and decide how much ride-pooling riders should pay~\cite{amano2020fair}.
In our context, we can formally define the Shapley value, $v_{i}$, as the marginal contribution of driver $i$ to each subset of drivers; that is, if we let $D=\{1, 2, \ldots, n\}$, and $\pi(s)$ be the income that a subset $s \in D$ would have made under the matching algorithm, for just those $s$ drivers. Then 
\begin{equation} 
    v_{i} = \sum_{s \subset D \setminus i} \frac{|s|! (n-|s|-1)!}{n!} (\pi(s \cup i)-\pi(s)).
\end{equation}
\begin{example}\label{ex:shapley}
Consider three drivers, and two riders, with prices \$10 and \$5. 
Suppose drivers 1 and 2 can service the first request, and drivers 2 and 3 can service the second request, where each driver can service at most one request. 
This results in the following profits for each subsets of drivers:

\begin{center}
    \begin{tabular}{ll}
    Driver Subset & Total Income \\
    \{\}          & 0            \\
    \{1\}         & 10           \\
    \{2\}         & 10           \\
    \{3\}         & 5            \\
    \{1,2\}       & 15           \\
    \{1,3\}       & 15           \\
    \{2,3\}       & 15           \\
    \{1,2,3\}     & 15          
    \end{tabular}
\end{center}

To compute the Shapley value of driver 1, consider the 4 subsets excluding driver 1: \{\}, \{2\}, \{3\}, \{2,3\}, which make 0, 10, 5, and 15 respectively.
When adding driver 1 to each of these subsets, the total income increases by 10, 5, 10, and 0 respectively. 
We then average over these subsets, proportional to the number of permutations for each subset, to get \begin{equation}
    \frac{2!}{3!} \times 10 + \frac{1!}{3!} \times 5 + \frac{1!}{3!} \times 5 + \frac{2!}{3!} \times 0 = 5,
\end{equation} which is the Shapley value of the driver. 

\end{example}

Due to Shapley value properties, $\sum_{i=1}^{n} v_{i} = \sum_{i=1}^{n} \pi_{i}$, or total value equals total income.  

Calculating the Shapley value requires enumerating all $O(2^{n})$ subsets of drivers. 
To reduce this, we approximate $v_{i}$ in $O(n)$ evaluations through a Monte Carlo simulation~\cite{ghorbani2019data}. 
Using the Shapley value of a driver, we redistribute income to reduce the difference between a driver's pre-redistribution income, $\pi_{i}$, and their value, $v_{i}$. 
The redistribution amount depends on a risk parameter, $0 \leq r \leq 1$, which designates what fraction of their income is kept by drivers. 
We collect $\sum_{i=1}^{n} (1-r) \pi_{i}$ from all drivers, and redistribute it proportional to the difference between their value and earnings, which is $\mathrm{max}(0,v_{i}-r\pi_{i})$. 
After redistribution, each driver earns $q_{i}$, defined by
\begin{equation}
    q_{i} = rv_{i} + \frac{\max (0,v_{i}-r\pi_{i})}{\sum_{j=1}^{n} \max (0,v_{j}-r\pi_{j})} \sum_{j=1}^{n} (1-r) v_{j}.
\end{equation}
This essentially allows drivers to keep some income, and redistributes the rest proportional to the difference between earnings and value, which allows drivers who earned less than their value to recoup some of their losses. 

\subsection{Redistribution Properties}\label{sec:income-redistribution}

Our approach to income redistribution begets two attractive theoretical properties. 
\begin{restatable}{thm}{thmmax}\label{thm:max}
When $\pi_{i}$, $\sum_{j=1}^{n} \pi_{j}$, and $r$ are constant, $q_{i}$ is maximized when $v_{i}$ is maximized. 
\end{restatable}

\begin{restatable}{thm}{thmmin}\label{thm:min}
For a given $v_{i}$, $r$, and fixed $\sum_{j=1}^{n} \pi_{j}$, the minimum value of $q_{i}$ is $\mathrm{min}(rv_{i},(1-r)v_{i})$.
\end{restatable}
Theorem~\ref{thm:max} incentivizes drivers to maximize their value, meaning that drivers are not better off putting in less effort under income redistribution.  Theorem~\ref{thm:min} gives a guarantee on the minimum amount that drivers can make, which is helpful in light of the recent debates over minimum wage for ride-pooling drivers~\cite{keeton2016uber}.  

\subsection{Experiments}\label{sec:income-exp-setup}
To test the effect of income redistribution, we vary the risk parameter and determine its effect on the income distribution.
We run experiments using a variety of objective functions and number of drivers, though we focus on $200$ drivers with the requests objective function when describing results. 
While we previously proved that drivers will aim to maximize $v_{i}$, making $v_{i}$ positively correlated with $q_{i}$, we did not specify the magnitude of this correlation.
To quantify the correlation, we define the gain metric, $g_{i}$, as the ratio of change in $q_{i}$ to $v_{i}$ when $v_{i}$ is doubled.
Because higher $v_{i}$ leads to higher $q_{i}$, we know $g_{i} \geq 0$, and ideally, if $v_{i}$ for a driver is doubled, then so would $q_{i}$, meaning $g_{i} = 1$. 
We calculate $g=\frac{1}{n} \sum_{i=1}^{n} g_{i}$ as the average gain metric over all drivers.

To determine the fairness of the income redistribution algorithm, we look at the distribution of the ratio of $v_{i}$ to $q_{i}$. 
Ideally, each driver would earn their value, and so the spread of the distribution would be 0.
However, as risk tolerance increases, less money will be redistributed, and so the spread will increase. 
To quantify this, we compute $\frac{q_{i}}{v_{i}}$ for all $i$, and then compute the standard deviation.

\subsection{Experiment Results}\label{sec:income-exp-results}
We outline some of the findings from our experiments: 
\begin{enumerate}
    \item \textbf{Ability to fairly redistribute.} We find that, for $0.5 \leq r \leq 0.9$, we can keep the gain larger than 0, while keeping the spread of income to value ratio near 0 (Figure~\ref{fig:shapley}). 
    In practice, this means when $r=0.9$, which corresponds to $g=0.8$, drivers receive, on average, an 80\% raise in income if they doubled their value. 
    At the same time, the standard deviation of $\frac{q_{i}}{v_{i}}$ is low, and so most drivers earn close to their value.  
    This allows for an equitable distribution of income while avoiding the free-rider problem, where income is spread out despite all the value being concentrated in one driver.
    
    \item \textbf{Tightness of Theoretical Bounds.} While the bounds in section~\ref{sec:income-redistribution} give results on the minimum amount that drivers are guaranteed to make, in reality, the worst-off drivers end up making significantly more than that. 
    For all values of $r$, the worst-off driver made at least \$2500, while the theoretical bounds peaked at \$1500. 
    The bounds assume the worst-case situation where one driver gets all the pre-redistribution income, however in practice, this rarely occurs, and so the bounds are not tight. 
\end{enumerate}

	\section{Conclusions \& Future Research}\label{sec:conclusions}
With an increasing ubiquity of ride-pooling services comes the need to critically evaluate fairness for both sides of the market -- riders and drivers alike. 
By modifying state-of-the-art ride-pooling matching algorithms, we explored objective functions that increase the number of requests serviced for certain numbers of drivers, and also increased fairness across groups. 
We additionally proposed addressing income inequality by using income redistribution to allocate income and alleviate fluctuation in income between drivers.  
By using income redistribution, we found that it is possible to avoid free-riding while keeping income inequality low. 

We discuss three potential avenues for future work:
\begin{enumerate}
    \item By evaluating methods on other data sets, we can test the generalizability of our methods and determine the circumstances under which the methods perform optimally.
    \item By developing provable guarantees for objective functions, we increase the robustness of our objective functions and guarantee that these objective functions work under different circumstances.
    \item By exploring income redistribution in other resource allocation settings, we can tackle the fairness versus profitability trade-off for other problems.
\end{enumerate}

\section*{Acknowledgements}
NSF CAREER IIS-1846237, NIST MSE \#20126334, DARPA \#HR00112020007, DARPA \#S4761, and DoD WHS \#HQ003420F0035.
	\clearpage
	\pagebreak
	 \section*{Ethical Impact}
We address a specific application of the classic fairness-efficiency trade-off found in numerous economic systems.
While we did derive our proposed definitions of fairness from reports of inequality in ride-pooling systems~\cite{Dillahunt17:Uncoering,Brown18:Ridehail,Pandey20:Iterative}, we acknowledge that 
may not be the proper solution for all settings where ride-pooling is deployed.  
Indeed, measuring, defining, and incorporating definitions of fairness into automated systems is an area of extremely active research within the FATE community. 
We confidently make the prescriptive statement that considerations of fairness at \emph{both} the driver and rider level \emph{should} be taken into account in some way in ride-pooling platforms.
However, open dialogue with stakeholders---e.g., local governments, riders and drivers hailing from various backgrounds---is imperative to understanding the desires of those stakeholders.

	{\small
    \bibliographystyle{named}
    \bibliography{references/references}

\begin{thebibliography}{}

\bibitem[\protect\citeauthoryear{Alonso-Mora \bgroup \em et al.\egroup
  }{2017}]{alonso2017demand}
Javier Alonso-Mora, Samitha Samaranayake, Alex Wallar, Emilio Frazzoli, and
  Daniela Rus.
\newblock On-demand high-capacity ride-sharing via dynamic trip-vehicle
  assignment.
\newblock {\em Proceedings of the National Academy of Sciences},
  114(3):462--467, 2017.

\bibitem[\protect\citeauthoryear{Amano \bgroup \em et al.\egroup
  }{2020}]{amano2020fair}
Yuki Amano, Ayumi Igarashi, Yasushi Kawase, Kazuhisa Makino, and Hirotaka Ono.
\newblock Fair ride allocation on a line.
\newblock {\em arXiv preprint arXiv:2007.08045}, 2020.

\bibitem[\protect\citeauthoryear{Bok{\'a}nyi and
  Hann{\'a}k}{2020}]{bokanyi2019ride}
Eszter Bok{\'a}nyi and Anik{\'o} Hann{\'a}k.
\newblock Ride-share matching algorithms generate income inequality.
\newblock {\em Scientific Reports}, 10(1):1--11, 2020.

\bibitem[\protect\citeauthoryear{Brown}{2018}]{Brown18:Ridehail}
Anne~Elizabeth Brown.
\newblock {\em Ridehail revolution: Ridehail travel and equity in Los Angeles}.
\newblock PhD thesis, UCLA, 2018.

\bibitem[\protect\citeauthoryear{Cook \bgroup \em et al.\egroup
  }{2018}]{cook2018gender}
Cody Cook, Rebecca Diamond, Jonathan Hall, John~A List, and Paul Oyer.
\newblock The gender earnings gap in the gig economy: Evidence from over a
  million rideshare drivers.
\newblock Technical report, National Bureau of Economic Research, 2018.

\bibitem[\protect\citeauthoryear{Dillahunt \bgroup \em et al.\egroup
  }{2017}]{Dillahunt17:Uncoering}
Tawanna~R Dillahunt, Vaishnav Kameswaran, Linfeng Li, and Tanya Rosenblat.
\newblock Uncovering the values and constraints of real-time ridesharing for
  low-resource populations.
\newblock In {\em Conference on Human Factors in Computing Systems (CHI)},
  pages 2757--2769, 2017.

\bibitem[\protect\citeauthoryear{Ghorbani and Zou}{2019}]{ghorbani2019data}
Amirata Ghorbani and James Zou.
\newblock Data {S}hapley: Equitable valuation of data for machine learning.
\newblock In {\em International Conference on Machine Learning (ICML)}, 2019.
\newblock Full version: arXiv:1904.02868.

\bibitem[\protect\citeauthoryear{Graham}{2017}]{graham2017towards}
Mark Graham.
\newblock {\em Towards a fairer gig economy}.
\newblock Meatspace Press, 2017.

\bibitem[\protect\citeauthoryear{Jia \bgroup \em et al.\egroup
  }{2019}]{jia2019towards}
Ruoxi Jia, David Dao, Boxin Wang, Frances~Ann Hubis, Nick Hynes, Nezihe~Merve
  Gurel, Bo~Li, Ce~Zhang, Dawn Song, and Costas Spanos.
\newblock Towards efficient data valuation based on the shapley value.
\newblock In {\em Conference on Artificial Intelligence and Statistics
  (AISTATS)}, 2019.

\bibitem[\protect\citeauthoryear{Keeton}{2016}]{keeton2016uber}
Richard~B Keeton.
\newblock An {U}ber dilemma: The conflict between the {S}eattle rideshare
  ordinance, the {NLRA}, and for-hire driver worker classification.
\newblock {\em Gonzaga Law Review}, 52:207, 2016.

\bibitem[\protect\citeauthoryear{Lesmana \bgroup \em et al.\egroup
  }{2019}]{lesmana2019balancing}
Nixie~S Lesmana, Xuan Zhang, and Xiaohui Bei.
\newblock Balancing efficiency and fairness in on-demand ridesourcing.
\newblock In {\em Advances in Neural Information Processing Systems (NeurIPS)},
  pages 5309--5319, 2019.

\bibitem[\protect\citeauthoryear{Li \bgroup \em et al.\egroup
  }{2019}]{li2019efficient}
Minne Li, Zhiwei Qin, Yan Jiao, Yaodong Yang, Jun Wang, Chenxi Wang, Guobin Wu,
  and Jieping Ye.
\newblock Efficient ridesharing order dispatching with mean field multi-agent
  reinforcement learning.
\newblock In {\em The World Wide Web Conference}, pages 983--994, 2019.

\bibitem[\protect\citeauthoryear{Lin \bgroup \em et al.\egroup
  }{2018}]{lin2018efficient}
Kaixiang Lin, Renyu Zhao, Zhe Xu, and Jiayu Zhou.
\newblock Efficient large-scale fleet management via multi-agent deep
  reinforcement learning.
\newblock In {\em Conference on Knowledge Discovery \& Data Mining (KDD)},
  pages 1774--1783. ACM, 2018.

\bibitem[\protect\citeauthoryear{Ma and Xu}{2020}]{ma2020group}
Will Ma and Pan Xu.
\newblock Group-level fairness maximization in online bipartite matching.
\newblock {\em arXiv preprint arXiv:2011.13908}, 2020.

\bibitem[\protect\citeauthoryear{Moody \bgroup \em et al.\egroup
  }{2019}]{moody2019rider}
Joanna Moody, Scott Middleton, and Jinhua Zhao.
\newblock Rider-to-rider discriminatory attitudes and ridesharing behavior.
\newblock {\em Transportation Research Part F: Traffic Psychology and
  Behaviour}, 62:258--273, 2019.

\bibitem[\protect\citeauthoryear{Nanda \bgroup \em et al.\egroup
  }{2020}]{nanda2020balance}
Vedant Nanda, Pan Xu, Karthik~Abinav Sankararaman, John~P Dickerson, and
  Aravind Srinivasan.
\newblock Balancing the tradeoff between profit and fairness in rideshare
  platforms during high-demand hours.
\newblock In {\em Conference on Artificial Intelligence (AAAI)}, 2020.

\bibitem[\protect\citeauthoryear{{New York City}}{2016}]{NYYellowTaxi2016}
{New York City}.
\newblock New {Y}ork {Y}ellow {T}axi data, 2016.

\bibitem[\protect\citeauthoryear{Pandey and
  Caliskan}{2020}]{Pandey20:Iterative}
Akshat Pandey and Aylin Caliskan.
\newblock Iterative effect-size bias in ridehailing: Measuring social bias in
  dynamic pricing of 100 million rides.
\newblock {\em arXiv preprint arXiv:2006.04599}, 2020.

\bibitem[\protect\citeauthoryear{Shah \bgroup \em et al.\egroup
  }{2020}]{shah2019neural}
Sanket Shah, Meghna Lowalekar, and Pradeep Varakantham.
\newblock Neural approximate dynamic programming for on-demand ride-pooling.
\newblock In {\em Conference on Artificial Intelligence (AAAI)}, 2020.
\newblock Full version: arXiv:1911.08842.

\bibitem[\protect\citeauthoryear{Shapley}{1953}]{Shapley53:Value}
Lloyd~S Shapley.
\newblock A value for $n$-person games.
\newblock {\em Contributions to the Theory of Games}, 2(28):307--317, 1953.

\bibitem[\protect\citeauthoryear{S{\"u}hr \bgroup \em et al.\egroup
  }{2019}]{suhr2019two}
Tom S{\"u}hr, Asia~J Biega, Meike Zehlike, Krishna~P Gummadi, and Abhijnan
  Chakraborty.
\newblock Two-sided fairness for repeated matchings in two-sided markets: A
  case study of a ride-hailing platform.
\newblock In {\em Conference on Knowledge Discovery \& Data Mining (KDD)},
  pages 3082--3092. ACM, 2019.

\bibitem[\protect\citeauthoryear{Turakhia}{2017}]{turakhia12engineering}
Chintan Turakhia.
\newblock Engineering more reliable transportation with machine learning and
  {AI} at {U}ber.
\newblock \url{https://eng.uber.com/machine-learning/}, 2017.
\newblock Accessed: 2021-01-20.

\bibitem[\protect\citeauthoryear{Uber}{2015}]{uber2015chicago}
Uber.
\newblock Chicago: An {U}ber case study.
\newblock
  \url{https://uber-static.s3.amazonaws.com/web-fresh/legal/Uber_Chicago_CaseStudy.pdf},
  2015.
\newblock Accessed: 2021-01-20.

\bibitem[\protect\citeauthoryear{Uber}{2020}]{matching}
Uber.
\newblock Matching: How does {U}ber match riders with drivers?
\newblock {\em URL: https://marketplace.uber.com/matching}, 2020.

\bibitem[\protect\citeauthoryear{Xu and Xu}{2020}]{xu2020trade}
Yifan Xu and Pan Xu.
\newblock Trade the system efficiency for the income equality of drivers in
  rideshare.
\newblock In {\em International Joint Conference on Artificial Intelligence
  (IJCAI)}, 2020.

\bibitem[\protect\citeauthoryear{Yan and Howe}{2020}]{yan2020fairness}
An~Yan and Bill Howe.
\newblock Fairness-aware demand prediction for new mobility.
\newblock In {\em Conference on Artificial Intelligence (AAAI)}, volume~34,
  pages 1079--1087, 2020.

\bibitem[\protect\citeauthoryear{Zhao \bgroup \em et al.\egroup
  }{2019}]{zhao2019preference}
Boming Zhao, Pan Xu, Yexuan Shi, Yongxin Tong, Zimu Zhou, and Yuxiang Zeng.
\newblock Preference-aware task assignment in on-demand taxi dispatching: An
  online stable matching approach.
\newblock In {\em Conference on Artificial Intelligence (AAAI)}, volume~33,
  pages 2245--2252, 2019.

\end{thebibliography}
	}

	\clearpage
	\pagebreak
	\clearpage
\appendix

\section{Matching Details}
We generate feasible combinations where the pickup delay is less than 300 seconds and the drop-off delay is less than 60 seconds, which follows from previous work~\cite{alonso2017demand}.

\section{Hyperparameters and Value Function Training}
We compare based on both the profitability metrics and the fairness metrics, while varying the total number of riders at full, half, and quarter demand, and the number of drivers at $10$, $50$, $100$, and $200$. 
We test each of the fairness objective functions with a variety of hyperparameters, $\lambda$, and set $\delta=5$, $|L|=4461$, $m_{i}=4$, and $H=10$. 
We train using data from the week of March 26th, and we list hyperparameters for each objective function below: 
\begin{enumerate}
    \item \textbf{Request}: We run the request objective function by training for 3 days and testing for 1 day.  
    \item \textbf{Driver-side fairness}: We run the driver-side fairness objective function by training for 3 days and testing for 1 day. We use $\lambda=\{0,\frac{1}{6},\frac{2}{6},\frac{3}{6},\frac{4}{6},\frac{5}{6},\frac{6}{6}\}$
    \item \textbf{Rider-side fairness}: We run the rider-side fairness objective function by training for 2 days and testing for 1 day. 
    We use $\lambda=\{10^{8},10^{9},10^{10}\}$
    \item \textbf{Income}: We run the income objective function by training for 3 days and testing for 1 day.  
\end{enumerate}

For both the objective function and redistribution experiments, we used a Red Hat Enterprise 7.9 server with 16 CPUs on an Intel Xeon X5550 processor, with 36 GB of RAM shared across multiple users.
Each simulation took no longer than 1 day, and many ran within a couple of hours. 
\\
For figure 2, we set $\lambda=\frac{4}{6}$ for driver-side fairness, and $\lambda=10^{9}$ for rider-side fairness. 




\section{Theoretical Guarantees}\label{app:theory}
In this section, we provide supporting proofs for our theoretical guarantees detailed in the income redistribution section. 
First, we recall Theorem~\ref{thm:max}, and provide its proof now.
\thmmax*
\begin{proof}
For constant $\pi_{i}$ and $r$, this is the same as maximizing 
\begin{equation}
    rv_{i} + \frac{\max (0,v_{i}-r\pi_{i})}{\sum_{j=1}^{n} \max (0,v_{j}-r\pi_{j})} \sum_{j=1}^{n} (1-r) v_{j}
\end{equation}
As $v_{i}$ increases, $rv_{i}$ also increases.
Similarly, increasing $v_{i}$ increases $\frac{\mathrm{min}(0,v_{i}-ra_{i})}{\sum_{i=1}^{n} \mathrm{min}(0,v_{i}-ra_{i})}$ and brings it closer to 1. 
Because $\pi_{i}$ is constant, then $\sum_{j=1}^{n} \pi_{j} = \sum_{j=1}^{n} v_{j}$ is also constant. 
Therefore, maximizing both the first and second terms is done through maximizing $v_{i}$, and so in order to maximize $q_{i}$, we need to maximize $v_{i}$. 
\end{proof}

Next, we recall Theorem~\ref{thm:min}, and provide its proof below.
\thmmin*
\begin{proof}
Consider $q_{i}$ as a function of $\pi_{i}$. 
Define $d=\sum_{i=1}^{n} \mathrm{max}(0,v_{i}-r\pi_{i})$, and $T=\sum_{i=1}^{n} v_{i}$, $d \leq T$. 
We rewrite $q_{i}$ as 
\begin{equation}
    rv_{i} + \frac{\mathrm{max}(0,v_{i}-ra_{i})}{d+\mathrm{max}(0,v_{i}-ra_{i})} * T*(1-r)
\end{equation}
This function has no local minima when $0 \leq \pi_{i} \leq v_{i}$, and when $\pi_{i} \geq v_{i}$, $q_{i} \geq rv_{i}$. 
At $\pi_{i}=0$, $q_{i} = \frac{v_{i}}{d} \times (1-r)T \geq \frac{v_{i}}{T} \times (1-r)T = (1-r)v_{i}$. 
Therefore, the minimum of $q_{i}$ is the smaller of these two, which is $\mathrm{min}(rv_{i},(1-r)v_{i})$.
\end{proof}

\section{Additional Experimental Results}
\begin{figure}[h!]
\includegraphics[width=\linewidth]{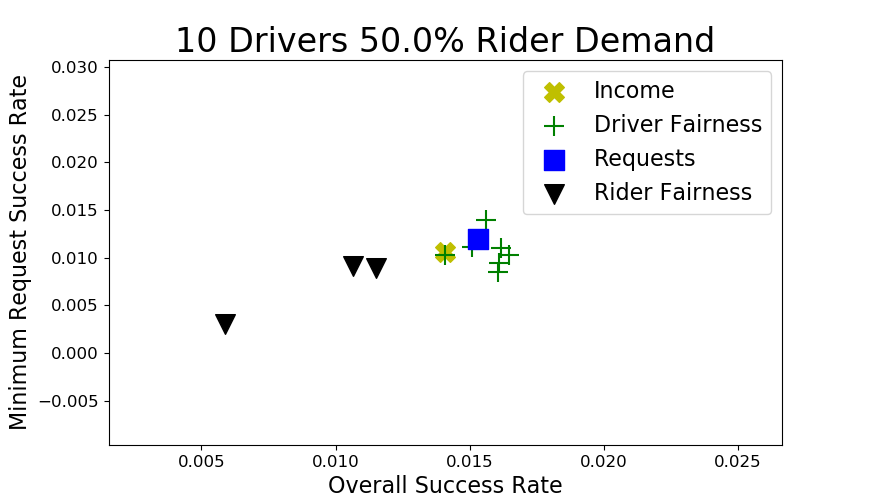}
\caption{We compare the overall success rate and the success rate in the worst-off neighborhood for 10 drivers when rider demand is 50\% of normal. We find that both requests and driver-fairness objective functions perform optimally.}%
\label{fig:10_5}
\end{figure}

\begin{figure}[h!]
\includegraphics[width=\linewidth]{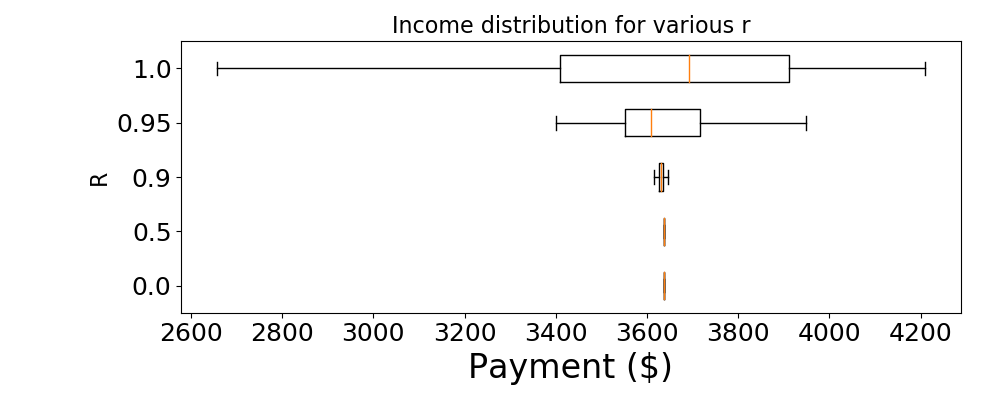}
\caption{We ran an experiment where $v_{i}$ was uniform for all drivers, with $n=200$ using the requests objective function. Using this definition, even for $r=0.9$, the spread of income is very small, and so is another way to increase earnings for those with lower wages. }%
\label{fig:distro_uniform}
\end{figure}

\begin{figure}[h!]
\includegraphics[width=\linewidth]{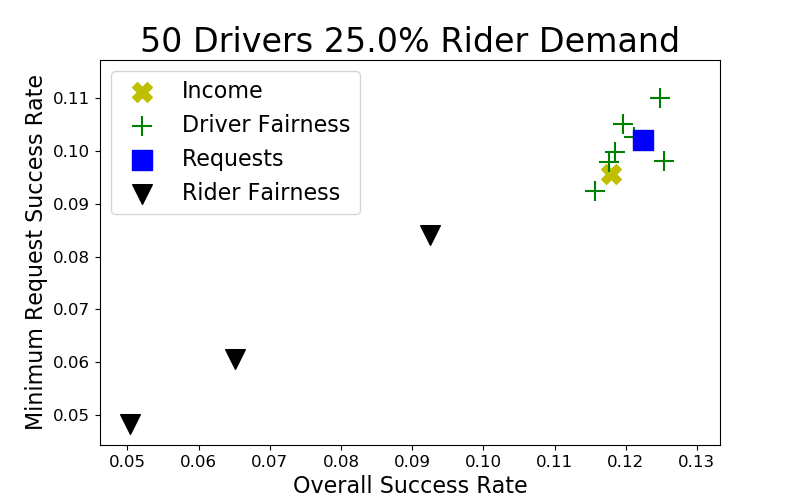}
\caption{We compare the overall success rate and the success rate in the worst-off neighborhood for 50 drivers when rider demand is 25\% of normal. We find that driver-fairness objective functions perform optimally, though the gap between the driver-fairness and requests objective functions tightens, when compared with figure \ref{fig:service}. This is because driver-fairness policies perform better when the number of drivers is much larger than the number of riders, and so when the number of riders decreases, the driver-fairness objective function does worse relative to the requests objective function.}%
\label{fig:50_25}
\end{figure}

\begin{figure}[h!]
\includegraphics[width=\linewidth]{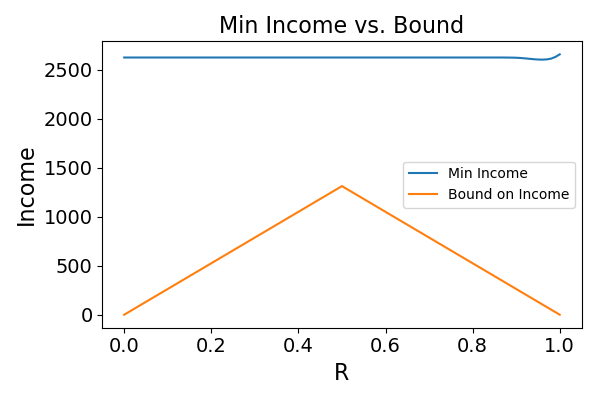}
\caption{We compare the theoretical bound on the minimum wage against the lowest income for any driver, at $n=200$ using the requests objective function. We find that the bound is very loose, as the worst-off driver makes nearly double the theoretical bound. }%
\label{fig:bounds}
\end{figure}

\begin{figure}[h!]
\includegraphics[width=\linewidth]{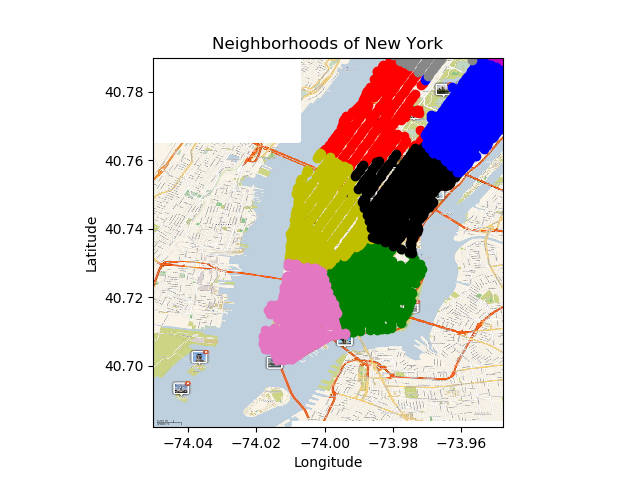}
\caption{We show a map of New York City, with each neighborhood in a different color (note that all 10 neighborhoods are not shown). }%
\label{fig:neighborhoods}
\end{figure}

\end{document}